\documentclass{article}
\usepackage{graphicx} 
\usepackage[noadjust]{cite}
\usepackage{latexsym,amssymb,enumerate,amsmath,epsfig,amsthm,authblk,dsfont,fullpage}
\usepackage{arydshln}
\usepackage{multirow}
\usepackage{setspace,color,soul}
\usepackage{tikz}
\usepackage{algorithm,algorithmic}
\usetikzlibrary{shapes,arrows}
\usepackage[normalem]{ulem}
\usepackage{textcomp}
\usepackage{tikz-cd}
\usepackage{accents}
\usepackage{amsmath, amsthm} 
\usepackage{xcolor} %
\usepackage{appendix}
\usepackage{graphicx,subfigure}
\usepackage{makecell}
\usepackage{tablefootnote} 
\usepackage{grffile}
\usetikzlibrary{decorations.pathreplacing}
\usepackage{wrapfig}
\usepackage{hyperref}
\usepackage{ulem}
\hypersetup{
	colorlinks=true,
	linkcolor=blue,
	citecolor=blue,
	filecolor=blue,      
	urlcolor=blue,
}
\allowdisplaybreaks

\newcommand{\calF}{{\mathcal F}}

\newcommand{\calW}{{\mathcal W}}
\newcommand{\calD}{{\mathcal D}}

\newcommand{\calB}{{\mathcal B}}
\newcommand{\calY}{\mathcal Y}
\newcommand{\calL}{\mathcal L}

\newcommand{\calA}{\mathcal A}

\newcommand{\calC}{\mathcal C}
\newcommand{\R}{{\mathbb R}}
\newcommand{\E}{\mathbb E}
\newcommand{\N}{\mathbb N}
\newcommand{\Prob}{\mathrm{Prob}}
\numberwithin{equation}{section}
\title{The Exploration of Error Bounds in Classification with Noisy Labels}
\author{ 
Haixia Liu\footnote{Part of work done during a visit at 
The Hong Kong University of Science and Technology.} \thanks{School of Mathematics and Statistics  \& Institute of Interdisciplinary Research for Mathematics and Applied Science \& Hubei Key Laboratory of Engineering Modeling and Scientific Computing, Huazhong University of Science and Technology, Wuhan, Hubei, China. Email: liuhaixia@hust.edu.cn. The work of H.X. Liu was supported in part by Interdisciplinary Research Program of HUST 2024JCYJ005, National Key Research and Development Program of China 2023YFC3804500.}, Boxiao Li\thanks{School of Mathematics and Statistics, Huazhong University of Science and Technology, Wuhan, Hubei, China. Email: liboxiao@hust.edu.cn.}, Can Yang\thanks{Department of Mathematics, The Hong Kong University of Science and Technology, Clear Water Bay, Hong Hong. Email:macyang@ust.hk. The work of Dr. Can Yang was supported in part by Hong Kong Research Grants Council Grants 16308120, 16307221, 16307322, 16302823 and 16309424; The Hong Kong University of Science and Technology Startup Grants R9405 and Z0428 from the Big Data Institute.}, Yang Wang\thanks{Corresponding author. Department of Mathematics, The University of Hong Kong, Pokfulam Road, Hong Hong. Email:yang.wang@hku.hk.}}
\date{\today}
\newtheorem{lemma}{Lemma}[section] 
\newtheorem{theorem}{Theorem}[section]
\newtheorem{assumption}{Assumption}[section]
\newtheorem{definition}{Definition}[section]

\newtheorem{remark}[theorem]{Remark}

\begin{document}

\maketitle
\begin{abstract}
Numerous studies have shown that label noise can lead to poor generalization performance, negatively affecting classification accuracy. Therefore, understanding the effectiveness of classifiers trained using deep neural networks in the presence of noisy labels is of considerable practical significance. In this paper, we focus on the error bounds of excess risks for classification problems with noisy labels within deep learning frameworks. We derive error bounds for the excess risk, decomposing it into statistical error and approximation error. To handle statistical dependencies (e.g., mixing sequences), we employ an independent block construction to bound the error, leveraging techniques for dependent processes. For the approximation error, we establish these theoretical results to the vector-valued setting, where the output space consists of $K$-dimensional unit vectors. Finally, under the low-dimensional manifold hypothesis, we further refine the approximation error to mitigate the impact of high-dimensional input spaces.
\end{abstract}
\section{Introduction}

Deep learning has made breakthroughs in the fields of computer vision \cite{Evan2017}, natural language processing \cite{Young2018}, medical image processing \cite {sarvamangala2022} and robotics \cite {bo2022}, which proves its strong efficiency in solving complicated problems. Leveraging high-performance computing resources and high-quality data, neural network-based methods outperform traditional machine learning approaches across various applications, including image classification, object detection, and speech recognition.

Over the years, the application of deep neural networks to classification problems has garnered significant attention. In practical applications, the data we collect may be disturbed by noise. Literature typically classifies noise into two types: attribute noise and class noise \cite{Zhu2004}. Attribute noise refers to disturbances that affect feature values, such as the addition of small Gaussian noise during measurement. In contrast, class noise impacts the labels assigned to instances, directly affecting the sample labels rather than the features. Some studies have indicated that random class noise can be more detrimental than random attribute noise, highlighting the importance of addressing class noise \cite{Jos2014}.

Labeling large-scale datasets is a costly and error-prone process, and even high-quality datasets are inevitably subject to mislabeling errors. Labels that may contain errors are referred to as noisy labels, which can arise from various sources. Due to the high cost of labeling, some labels are obtained from non-expert sources on websites, such as image search engines or web crawlers. Even data labeled by experts may contain errors; these classification mistakes are not solely due to human oversight but can also result from the widespread use of automated classification systems. Additionally, the limitations of descriptive language can contribute to the occurrence of noisy labels \cite{brazdil1990learning}. Furthermore, some artificial noises are introduced to protect individuals' privacy \cite{van2002, ghazi2021,liu2026rpwithprior,liu2026blockrr}.

Numerous studies have demonstrated that label noise can lead to poor generalization performance, negatively impacting classification accuracy \cite{Frenay, David2010}. Therefore, the effectiveness of classifiers trained using deep neural networks in the presence of noisy labels is of considerable practical importance.

Although neural networks have achieved remarkable results in real applications, our theoretical understanding of them remains limited. The expressive power of neural networks refers to their ability to approximate various functions. Understanding the approximation capabilities of deep neural networks has become a fundamental research direction for uncovering the advantages of deep learning over traditional methods, with related studies dating back to at least the 1980s. The Universal Approximation Theorem \cite{Andrew1994} states that a neural network with an appropriate activation function and a depth of two can approximate any continuous function on a compact domain to any desired precision. This illustrates the powerful expressive capability of neural networks and provides a theoretical foundation for their application. However, in practice, this approximation theorem has significant limitations. For instance, approximating complex functions may require a large number of neurons, leading to computational difficulties.

Currently, research on error bounds for classification problems within deep learning frameworks remains relatively scarce, especially in the context of finite samples and label noise. This paper aims to fill this gap.  To estimate the error bound, we specifically divide them into approximation error and statistical error, which can be  estimated separately. Our contributions can be summarized as follows:\\
(1) \textbf{We provide an error bound for the classification problem with noisy labels.} The main results are in Theorem \ref{theo:excess_risk} and Theorem \ref{theo:curse}.\\ 
(2) \textbf{The statistical error:} In practice, the samples in the dataset may not be independent. In this paper, we focus on a dependent (mixing) sequence. We bound the statistical error with the help of the associated independent block (IB) sequence. \\
(3) \textbf{The approximation error:} While prior works primarily focused on scalar-valued functions (output space $\R^1$), we generalize these theoretical foundations to the vector-valued setting, where the output space consists of $K$-dimensional unit vectors. \\
(4) \textbf{The curse of dimensionality:} Based on the low-dimensional manifold assumption, we effectively alleviate the issue of the curse of dimensionality.

The rest of this paper is organized as follows. Section \ref{sec:related} reviews the related works about noisy labels and error estimations. Section \ref{sec:setup} presents the problem setup, including the definitions of the expected excess risk estimations on the true and noisy distributions, the empirical excess risk estimations on the true and noisy data, the neural networks we considered in this paper, which is followed by the theoretical results of the excess risk under noisy labels in Section \ref{sec:excess}. Some technical details of proofs are provided in Section \ref{sec:proof}. In Section \ref{sec:curse}, we present a theoretical result aimed at alleviating the curse of dimensionality. Lastly, we conclude the results in Section \ref{sec:conclusion}.

\section{Related Works}\label{sec:related}
\subsection{Noisy label}
Learning from noisy-labeled data is an active area of research. Learning discriminative models was formulated using undirected graphical models \cite{vahdat2017} or probabilistic graphical models \cite{xiao2015}, which represent the relationship between noisy and clean labels in a semi-supervised setting. Sukhbaatar et al. \cite{Sukhbaatar2015} introduced an additional noise layer in the network that adapted the outputs to align with the noisy label distribution. Patrini et al. \cite{patrini2017} proposed two procedures for loss correction that were agnostic to both the application domain and the network architecture.

Furthermore, decoupling methods were designed for training with noisy labels using two models. Malach et al. \cite{Malach2017} introduced a meta-algorithm to tackle the issue of noisy labels by training two networks and establishing update rules that decoupled the decisions of `when to update' and `how to update'. Following this, the Co-teaching algorithm \cite{han2018} was developed, enabling two deep neural networks to train concurrently while teaching each other under mini-batch conditions. Building on the Decoupling and Co-teaching algorithms, the Co-teaching+ algorithm \cite{Yu2019} was introduced. By implementing the strategy of `Update by Disagreement', this method prevented the two networks from becoming overly consistent during training.

Additionally, Lee et al. \cite{Lee2018} addressed noisy labels through transfer learning methods, introducing {\it CleanNet}, a joint neural embedding network that provided insights into label noise by manually verifying only a small subset of classes and passing this knowledge to other classes. Jiang et al. \cite{jiang2018} proposed an algorithm for joint optimization of large-scale data using {\it MentorNet}, which supervised the training of a basic deep network, {\it StudentNet}, enhancing its generalization ability when trained on noisy labels.

Xia et al. \cite{xia2020} introduced a robust early learning method that distinguished between critical and non-critical parameters, applying different update rules to gradually reduce non-critical parameters to zero, thereby suppressing the impact of noisy labels. Li et al. \cite{Li2022} proposed selective supervised contrastive learning (Sel-CL), which trained with confident and noisy pairs, enabling the model to distinguish between noise and mitigate its influence on supervised contrastive learning (Sup-CL). An alternative approach was to seek loss functions that were inherently noise-tolerant \cite{ghosh2017robust}. The authors provided sufficient conditions for a loss function to ensure that risk minimization under that function was inherently resilient to label noise in multiclass classification problems. Subsequently, symmetric cross-entropy was introduced for robust learning with noisy labels \cite{wang2019symmetric}. Wei et al. \cite{Wei2023} presented a framework for label noise in fine-grained datasets, stochastic noise-tolerated supervised contrastive learning (SNSCL), which mitigated the effects of noisy labels by designing noise-tolerant loss functions and incorporating random modules.

\subsection{Error estimation}

In recent years, research on the expressive and approximation capabilities of neural networks has become increasingly active. The approximation error of feedforward neural networks with various activation functions has been studied for different types of functions. Lu et al. \cite{zhou2017} examined the influence of width on the expressive ability of neural networks, establishing corresponding approximation error estimates for Lebesgue integrable functions. They obtained a general approximation theorem for ReLU networks with bounded width, demonstrating that a ReLU network with a width of $d + 4$ is a universal approximator, where $d$ is the input dimension. Lu et al. \cite{Lu2021} established the approximation error bounds of ReLU neural networks for smoothing functions concerning width and depth. They proved that a deep neural network with width $O(N\ln N)$ and depth $O(L\ln L)$ can approximate functions $f \in C^s([0,1]^d)$, with an approximation error of $O\left ( \left \| f \right \|_{C^s([0,1]^d)} N^{-2s/d}L^{-2s/d}\right ) $ for any $N,L\in\mathbb{N}^+$, where $ \left \| f \right \|_{C^s([0,1]^d)}:=\max\left \{ \left \| \partial ^{\alpha }f \right \|   _{L^\infty ([0,1]^d)}:\left \| \alpha \right \|_1 \le s,\alpha \in \mathbb{N}^d\right \} $. Petersen et al. \cite{PETERSEN2018} established the best approximation error bounds for approximating piecewise $C^\beta$ functions on the interval $[-1/2, 1/2]^d$ using ReLU neural networks. Hadrien et al. \cite{Hadrien2021} provided an error estimate for approximating generalized finite-bandwidth multivariate functions with deep ReLU networks. 
Collectively, these studies focus on the ability of ReLU-activated feedforward neural networks to approximate various types of functions.

Additionally, many studies have explored other activation functions. Danilo et al. \cite{COSTARELLI2013} investigated the approximation order of neural network operators with sigmoid activation functions using a moment-type approach. DeRyck et al. \cite{DERYCK2021732} provided error bounds for neural networks utilizing hyperbolic tangent activation functions in the approximation of Sobolev regular and analytic functions, proving that tanh networks with only two hidden layers can approximate functions at a faster or equivalent rate compared to deeper ReLU networks.

Another direction of research focuses on the error bounds of regression model in $\beta$-H\"older smooth space, which are established to evaluate approximation quality.

Stone \cite{Stone1982} assumed that the true regression model is $\beta$-H\"older smooth, where the smoothness index $\beta > 0$. The convergence rate of the prediction error is $C_{d}n^{-2\beta/(2\beta + d)}$. This indicates that the convergence rate is exponentially dependent on the dimension $d$ of the inputs, while the prefactor $C_d$ is independent of the sample size $n$ but depends on 
$d$ and other model parameters. Research by Schmidt-Hieber \cite{Schmidt-Hieber2020} and Nakada and Imaizumi \cite{Nakada2020} demonstrated that deep neural network regression can achieve the optimal minimax rate established by Stone in 1982 under certain conditions. Shen, Yang and Zhang \cite{shen2020} derived the approximate error bound for $\beta$-H\"older continuous functions when the smoothness index $\beta\le 1$, finding that the prefactor depends on a polynomial of $d$. Lu et al. \cite{Lu2021} studied the approximate error bound for $\beta$-H\"older continuous functions when the smoothness index 
$\beta\ge1$, noting that its prefactor relates to an exponential form of  $d$. Jiao et al. \cite{jiao2023} improved the result of \cite{Lu2021}. For the case of the smoothness index $\beta>1$, they found that the non-asymptotic upper limit of the prediction error of the empirical risk minimizer is polynomially dependent on the dimension $d$ of the predictor rather than exponentially dependent. Chen et al. \cite{2022Nonparametric} focused on nonparametric regression of H\"older functions on low-dimensional manifolds using deep ReLU networks. They proved the adaptability of deep ReLU networks to low-dimensional geometric structures in data, partially explaining their effectiveness in handling high-dimensional data with low-dimensional geometric structures. Feng et al. \cite{feng2023} recently addressed over-parameterized deep nonparametric regression. By effectively balancing approximation and generalization errors, they derived results for H\"older functions with constrained weights, highlighting their application in reinforcement learning.

Furthermore, numerous studies have concentrated on the curse of dimensionality, a phenomenon characterized by a significantly slowed convergence rate when the input dimension $d$ is excessively large. This challenge is inherent in high-dimensional computations, rendering it a pivotal area of investigation. 

Yang et al. \cite{Hadrien2021} focused on constraining the objective function space and demonstrated a theorem on the approximation of generalized bandlimited multivariate functions using deep ReLU networks, which overcame the curse of dimensionality. This theorem builds upon Maurey's work and leverages the proficiency of deep ReLU networks in efficiently approximating Chebyshev polynomials and analytic functions. 
Shen et al. \cite{shen2021} considered neural networks equipped with alternative activation functions introduced Floor-ReLU networks, achieving root exponential convergence while simultaneously avoiding the curse of dimensionality for (H\"older) continuous functions. These networks provide an explicit error bound in deep network approximation. 
Jiao et al. \cite{jiao2023} was grounded in the low-dimensional subspace assumption and illustrated that the neural regression estimator could bypass the curse of dimensionality under the premise that the predictor resides on an approximate low-dimensional manifold or a set possessing low Minkowski dimension.

\section{Setup of Classification Problem with Noisy Labels}\label{sec:setup}
We consider a $K$-class classification problem with noisy labels. We start from  the expected risk, which quantifies the average loss of a classifier over the entire population or distribution.
\begin{definition}
Let $Z = (X, Y,Y^\eta)$ be a tuple of random vectors, where $(X,Y)$ follow a distribution $D$, $X \in \R^d$ represents the feature vector,  $Y,Y^\eta\in\calY=\{y=[y_1\ \cdots\ y_K]^\top \in\R^K|y_i\in[0,1],\sum^K_{i=1}y_i=1\}$ denote the true label and its perturbed noisy version, respectively. Let $f$ be a map which is from $\R^d$ to $\calY$ and $\ell(\cdot,\cdot)$ be a loss function. Then the expected risks associated with true distribution or noisy distribution are defined as 
    \begin{align} \label{erisk}
    \calL(f)=\E_{X,Y}[\ell(f(X),Y)],\quad\calL^\eta(f)=\E_{X,Y^\eta}[\ell(f(X),Y^\eta)].
    \end{align}
\end{definition}
In practice, the distribution $D$ and its perturbed versions are often unknown. Consequently, we typically focus on exploring the empirical risks at the sample level.
\begin{definition}
    Let $(x,y,y^\eta)=\{(x_i, y_i,y^\eta_i)\}_{i=1}^n$ be the dataset, where $x_i$ represents the feature of the $i$-th data point, $y_i,y^\eta_i$ are its true label and noisy label. Let $f$ be a map from $\R^d$ to $\calY$, and let $\ell(\cdot,\cdot)$ denote a loss function. Then the empirical risks about the true data or the noisy data are
    \begin{equation}
    \calL_{n}(f)=\frac{1}{n} \sum_{i=1}^{n} \ell(f(x_i),y_i),\quad\calL^\eta_{n}(f)=\frac{1}{n} \sum_{i=1}^{n} \ell(f(x_i),y^\eta_i).
    \end{equation}
\end{definition}	
 A loss function, denoted $\ell(\cdot, \cdot)$, is commonly introduced to quantify the differences between $f(x)$ and its associated label $y$. This includes various loss functions such as the $\ell_1$ loss, the $\ell_2$ loss, the cross entropy (CE) loss and the reverse CE loss \cite{wang2019symmetric}, among others. 

In practice, it is infeasible to explore all possible functions to find the optimal classifier. In recent years, a series of research results on deep neural networks have been achieved with great success. We select a class of deep neural networks with a specific structure. Let $\calF_{d,K}(\mathcal{W},\mathcal{D})$ be a class of a deep neural network with an input dimension of $d$, an output dimension of $K$, width $\calW$ and depth $\calD$. In the following, we use $\mathcal{F}$ as a shorthand when the parameters used in the neural networks do not need to be emphasized.

In this paper, we specifically consider a class of the ReLU neural networks. Let $\phi(\cdot) \in \mathcal{F}_{d,K}(\mathcal{W},\mathcal{D})$ be a function that maps $[0,1]^d$ to $\R^K$.  
The function $\phi$ can be constructed as follows:
\begin{align*}
\phi_0(x) = x, \ 
\phi_{i+1}(x) = \sigma(W_i \phi_i(x) + b_i),\  i = 0, 1, \cdots, \mathcal{D}-1, \
\phi(x) = W_{\mathcal{D}} \phi_{\mathcal{D}}(x) + b_{\mathcal{D}}.
\end{align*}
The weight matrix is given by $W_i \in \mathbb{R}^{p_{i+1} \times p_i}$, and the bias vector is $b_i \in \mathbb{R}^{p_{i+1}}$, where $p_i$ is the width of the $i$-th layer for $i=0,\cdots,\calD$, the activation function is defined as $\sigma(x) = \max(x, 0)$. Let $\mathrm{softmax}:\R^K\to\calY=\{y=[y_1\ \cdots\ y_K]^\top \in\R^K|y_i\in[0,1],\sum^K_{i=1}y_i=1\}$ be the softmax function, then $\mathrm{softmax}\circ\phi: [0,1]^d \to \calY$.

The structure of the neural network above consists of $\mathcal{D}$ hidden layers and a total of $(\mathcal{D}+2)$ layers, with $p_0 = d$ and $p_{\mathcal{D}+1} = K$. The width of the neural network satisfies $\mathcal{W} = \max\{p_1, \cdots, p_{\calD}\}$, the size of the neural network is given by $\mathcal{Q} = \sum_{i=0}^{\mathcal{D}} p_{i+1} \times (p_i + 1)$, and the number of neurons is $\mathcal{U} = \sum_{i=1}^{\mathcal{D}} p_i$. 

As a subset of $\calF_{d,K}(\mathcal{W},\mathcal{D})$, a class of the norm-constrained neural network, denoted as $\mathcal{F}_{d,K}(\mathcal{W},\mathcal{D},\calB)$, represents the collection of functions $\phi(\cdot) \in \mathcal{F}_{d,K}(\mathcal{W},\mathcal{D})$ subject to the constraint:
\begin{equation*}
   \| (W_{\calD},b_{\calD}) \| \prod_{i=0}^{\calD-1} \max\left\{ \| (W_{i}, b_{i}) \| , 1 \right\} \leq \calB,
\end{equation*}
where $ \calB$ is a positive constant and the norm $\|\cdot\|$ can be chosen according to the specific requirement.

Throughout the paper, we use $f_0:=\arg \min_{f} \calL(f)$, $f^\eta_0:=\arg \min_{f} \calL^\eta(f)$ as the global minimizers of the expected risks of the true distribution and the noisy distribution. Similarly, we define  
\begin{equation}\label{eq:def_empirical_risk}
\hat{f_n} = \mathop{\arg \min}_{f \in \calF} \calL_{n}(f),\ \hat{f^\eta_n} = \mathop{\arg \min}_{f \in \calF} \calL^\eta_{n}(f),
\end{equation}
as the empirical risk minimizer for the true data and the noisy data, 
where $\calF$ represents a deep neural network class with a certain structure. Additionally, $\lfloor x\rfloor$ denotes the largest integer no greater than $x$, $\lceil x\rceil$ denotes the smallest integer no less than $x$, $x\lesssim y$ or $y\gtrsim x$ denotes $x\le Cy$ for some positive constant $C>0$ and $x\asymp y$ means $x\lesssim y\lesssim x$. For a multi-index $s=[s_1\cdots s_d]$ and a $d$-dimensional vector $x=[x_1\cdots x_d]\in\R^d$, $x^s=x^{s_1}_1\cdots x^{s_d}_d$. $\N$ and $\N^+$ are the non-negative and positive integer spaces, respectively.
\section{The Excess Risk under Noisy Labels}\label{sec:excess}
We consider the following classification problem
\[Y^\eta=n_\eta(Y)=n_\eta(f_0(X)),\]
where $X$ represents the features, $Y,Y^\eta$ are the true and noisy labels, respectively, $f_0(\cdot)$ is the unknown classifier, $n_\eta(\cdot)$ is a disturbing function which maps the true labels to the noisy labels. 

In this section, we explore the theoretical bounds about classification problem with noisy labels. We focus on the excess risks on the true and noisy distributions, which are defined as

\begin{equation*}
\begin{split}
    \calL(\hat{f}^{\eta}_n)-\calL(f_0)=&\E_{X,Y}[\ell(\hat{f}^{\eta}_n(X),Y)]-\E_{X,Y}[\ell(f_0(X),Y)],\\
    \calL^\eta(\hat{f}^{\eta}_n)-\calL^\eta(f_0)=&\E_{X,Y^\eta}[\ell(\hat{f}^{\eta}_n(X),Y^\eta)]-\E_{X,Y^\eta}[\ell(f_0(X),Y^\eta)].
\end{split}
\end{equation*}

For the true classifiers (unknown), we represent as $f_0=\mathrm{softmax}\circ \kappa:[0,1]^d\to\calY=\{y=[y_1\ \cdots\ y_K]^\top|y_i\ge 0\ \forall i, \sum^K_{i=1}y_i=1\}$, where $\kappa$ belongs to the following class of smooth maps:
 \begin{definition}\label{def:smooth_maps}
 We define the following class of smooth maps:
 \begin{equation*}
     \calC_{\tau,d,K,V}:=\left\{ f=[f_1\cdots f_K]^\top:f_i \in C^r\left ( [0,1]^d \right ) :\left\| f_i \right\| _{C^{0,\tau}}\le U,\forall i \right\} ,
\end{equation*}
 where $\tau \in (0, \infty)$ with $\tau = r + \omega$, $r \in \mathbb{N}$, $d \in \mathbb{N}^+$, $\omega
 \in (0,1]$ and the $C^{0,\tau}$-norm is defined by

 \begin{equation*}
\left \| g \right \| _{C^{0,\tau}}:=\max \left \{ \max_{\left |\alpha \right | \le r   }  { \sup_{x\in[0,1]^d}\left | \partial ^\alpha g \right |},\max_{\left | \alpha  \right | =r} \sup_{x,y\in [0,1]^d,x\ne y}\frac{\partial ^\alpha g (x)-\partial ^\alpha g (y)}{\left | x-y \right | ^{\omega }}  \right\},\ \forall g\in C^r\left ( [0,1]^d \right ).
\end{equation*}
 \end{definition}
In practice, the samples in a dataset may not be independent. In this paper, we focus on a dependent (mixing) sequence.
\begin{definition}[$\beta$-mixing,  Subsection 2.2 in \cite{Yu1994}]\label{def:ib}
 Let $\left \{ x_t \right \} _{t \geq 1}$ be a stochastic process, and define $x^{i:j} = (x_i, \dots, x_j)$ for $i\le j$, where $j$ can be infinite. Let $\sigma(x^{i:j})$ denote the $\sigma$-algebra generated by $x^{i:j}$ for $i \leq j$. The $s$-th $\beta$-mixing coefficient of $\{ x_t \}_{t \geq 1}$, denoted by $\beta_s$, is defined as
\begin{equation}\label{eq:beta-mixing}
    \beta_s=\sup_{t\ge 1}\E\left [ \sup_{B\in \sigma (x^{t+s:\infty})} \left |  \Prob(B|\sigma(x^{1:t}))-\Prob(B)\right | \right ].
\end{equation}
We say that $\{ x_t \}_{t \geq 1}$ is $\beta$-mixing if $\beta_s \to 0$ as $s \to \infty$. 
\end{definition}

\begin{theorem}\label{theo:excess_risk}
    Let   $f_0=\mathrm{softmax}\circ \kappa:[0,1]^d\to\calY=\{y=[y_1\ \cdots\ y_K]^\top|y_i\ge 0\ \forall i, \sum^K_{i=1}y_i=1\}$ and $\kappa\in\calC_{\tau,d,K,V}$. Let $S^\eta=\left \{ z^\eta_i=(x_i,y^\eta_i) \right \}_{i=1}^{n}$, $S=\left \{ z_i=(x_i,y_i) \right \}_{i=1}^{n}$ be  strictly stationary $\beta$-mixing sequences and $(a_n, \mu_n)$ be an integer pair such that \( n = a_n\mu_n \). For a large enough $\calB \in \mathbb{N}^+$, when $\calW\gtrsim \calB^{d/(d+1)} \log\calB$, $\calD\gtrsim \log\calB$ and $\left | \ell(\mathrm{softmax}(a),q)-\ell(\mathrm{softmax}(b),q) \right | 
       \le\lambda\|a-b\|_2,\ \forall a,b\in\R^K,\forall q\in\calY$, the excess risks for the empirical risk minimizers $\hat{f_n}, \hat{f^\eta_n} \in \calF_{d,K}(\calW,\calD,\calB)$ can be bounded as \begin{align}
     &\E_{S^\eta}[\calL^\eta(\hat{f^\eta_n})-\calL^\eta(f_0)]
     \lesssim \frac{8\lambda \sqrt{K}\calB \sqrt{\calD+2+\log d}}{\sqrt{na_n}}+\frac{4 \lambda  \sqrt{K}n\beta_{a_n}}{a_n}+\lambda\sqrt{K} \calB^{-\tau/(d+1)},\label{eq:bound1}\\
     &\E_{S}[\calL(\hat{f_n})-\calL(f_0)]
     \lesssim \frac{8\lambda \sqrt{K}\calB \sqrt{\calD+2+\log d}}{\sqrt{na_n}}+\frac{4 \lambda  \sqrt{K}n\beta_{a_n}}{a_n}+\lambda\sqrt{K} \calB^{-\tau/(d+1)},\label{eq:bound2}
\end{align}
where $\beta_{a_n}$ is the $a_n$-th $\beta$-mixing coefficient of samples defined in Definition \ref{def:ib}. 
\end{theorem}
\begin{proof}
    We defer the proof to Section \ref{sec:proof}.
\end{proof}
\begin{remark}\label{remark:condition}
   As verified in Appendix \ref{appendix:lipschitz}, the Lipschitz condition $$\left | \ell(\mathrm{softmax}(a),q)-\ell(\mathrm{softmax}(b),q) \right | 
       \le\lambda\|a-b\|_2,\ \forall a,b\in\R^K,\forall q\in\calY$$ holds for all the $\ell_p, p\ge1$ losses,  the CE loss and the reverse CE loss (see the definition of Defintion \ref{appendix:def_reversece}).

\end{remark}



In the proof of the error bound for the excess empirical risk minimizer, we dissect it into two components: the statistical error and the approximation error. On the right-hand side of inequalities \eqref{eq:bound1} and \eqref{eq:bound2}, the first two terms represent the upper bound of the statistical error, whereas the third term signifies the upper bound of the approximation error.

The second terms on the right-hand side of inequalities \eqref{eq:bound1} and \eqref{eq:bound2} arise due to the dependence among the data points. When the data points are independent, this dependency term diminishes, as the observations no longer influence each other. In such cases, the error bounds may simplify and potentially tighten, reflecting the reduced complexity due to the absence of interdependencies among the data. 

In addition, $\beta_{a_n}$ tends to zero as $a_n$
  approaches infinity. Thus, $\beta_{a_n}$ can be made small by choosing $a_n$ sufficiently large.

\section{Proof of Theorem \ref{theo:excess_risk}}\label{sec:proof}
In this subsection, we provide a detailed proof of Theorem \ref{theo:excess_risk}. In Subsection \ref{sub:error_analysis}, we estimate the error bounds of the excess risk for different loss functions, categorizing them into statistic error and approximation error. We analyze these two types of errors in Subsections \ref{sub:statistical} and \ref{sub:approx}, which are followed by the proof of Theorem \ref{theo:excess_risk}.

\subsection{Error analysis}
\label{sub:error_analysis}
In this subsection, we analyze the error bounds for excess risks on noisy and true data.
 
\begin{lemma}\label{lem:loss1}
Let $f_0=\mathrm{softmax}\circ\kappa$ represent the true classifier, and $\hat{f^\eta_n}$ denote the empirical risk minimizer on the noisy data.  Let $ | \ell(\mathrm{softmax}(a),q)-\ell(\mathrm{softmax}(b),q)  | 
       \le\lambda\|a-b\|_2,\ \forall a,b\in\R^K,\forall q\in\calY$. We have the following
\begin{equation}
    \calL^\eta(\hat{f^\eta_n})-\calL^\eta(f_0)
    \le 2\sup_{\substack{f=\mathrm{softmax}\circ \psi,\\
    \psi\in \mathcal{F}_{d,K}(\mathcal{W},\mathcal{D},\calB)} } \left |\calL^\eta(f)-\calL^\eta_n(f)  \right |  +\lambda\inf_{\phi \in \mathcal{F}_{d,K}(\mathcal{W},\mathcal{D},\calB)}\left \| \phi-\kappa \right \|_{L^2(\nu)},
\end{equation} 
where $\nu$ is the the marginal probability measure of $X$. 
\end{lemma}

\begin{proof}
Let $\tilde f=\mathrm{softmax}\circ\tilde{\kappa}$ with $\tilde{\kappa} \in \mathcal{F}_{d,K}(\mathcal{W},\mathcal{D},\calB)$ satisfying $\|\tilde \kappa-\kappa\|_{L^2(\nu)}\le\|\phi-\kappa\|_{L^2(\nu)},\ \forall \phi\in\mathcal{F}_{d,K}(\mathcal{W},\mathcal{D},\calB)$,  then 
\begin{align}
   & \label{eq:error decomposition1}
    \calL^\eta(\hat{f^\eta_n})-\calL^\eta(f_0)\nonumber\\
    =&\calL^\eta(\hat{f^\eta_n})-\calL^\eta_n(\hat{f^\eta_n})+\calL^\eta_n(\hat{f^\eta_n})-\calL^\eta_n(\tilde f)+\calL^\eta_n(\tilde f)-\calL^\eta(\tilde f)+\calL^\eta(\tilde f)-\calL^\eta(f_0)\nonumber\\
    \le& \calL^\eta(\hat{f^\eta_n})-\calL^\eta_n(\hat{f^\eta_n})+\calL^\eta_n(\tilde f)-\calL^\eta(\tilde f)+\calL^\eta(\tilde f)-\calL^\eta(f_0)\nonumber\\
    \le& 2\left[\sup_{\psi \in \mathcal{F}_{d,K}(\mathcal{W},\mathcal{D},\calB)} \left |\calL^\eta(\mathrm{softmax}\circ\psi)-\calL^\eta_n(\mathrm{softmax}\circ\psi)  \right | \right] +\left|\calL^\eta(\tilde f)-\calL^\eta(f_0)\right|.
\end{align}
Here, we utilize the fact that $\calL^\eta_n(\hat{f^\eta_n})\le\calL^\eta_n( f),\ \forall f\in\mathcal{F}_{d,K}(\mathcal{W},\mathcal{D},\calB)$ in the first inequality of \eqref{eq:error decomposition1}. In addition, 
     \begin{equation}
    \begin{aligned}\label{C.3}
    &|\calL^\eta(\tilde f)-\calL^\eta(f_0)|\\=&\E_{X,Y^\eta}\left [  \ell(\tilde f(X),Y^\eta)\right ] -\E_{X,Y^\eta}\left [  \ell(f_0(X),Y^\eta)\right ] \le \lambda \E_{X}\|\tilde \kappa(X)-\kappa(X)\|_2\\
    \le&\lambda\left \| \tilde \kappa-\kappa \right \|_{L^2(\nu )}=\lambda\inf_{\phi \in \mathcal{F}_{d,K}(\mathcal{W},\mathcal{D},\calB)}\left \| \phi-\kappa \right \|_{L^2(\nu)}.
    \end{aligned}
    \end{equation}   
 Combining \eqref{eq:error decomposition1},  and \eqref{C.3}, we conclude the result.
\end{proof}
For the excess risk about the empirical risk minimizer on the true data $\calL(\hat{f^\eta_n})-\calL(f_0)$, we have the following similar results.
\begin{lemma}\label{lem:loss}
Let $f_0=\mathrm{softmax}\circ\kappa:[0,1]^d\to\calY=\{y=[y_1\ \cdots\ y_K]^\top|y_i\ge0,\sum^K_{i=1}=1\}$ represent the true classifier, and $\hat{f_n}$ denote the empirical risk minimizer on the true data.  Let $ | \ell(\mathrm{softmax}(a),q)-\ell(\mathrm{softmax}(b),q)  | 
       \le\lambda\|a-b\|_2,\ \forall a,b\in\R^K,\forall q\in\calY$, the excess risks for the empirical risk minimizer $\hat{f_n}=\mathrm{softmax}\circ\hat{\kappa}_n$ with $\hat{\kappa}_n \in \calF_{d,K}(\calW,\calD,\calB)$. We have the following
\begin{equation}
    \calL(\hat{f}_n)-\calL(f_0)
    \le 2\sup_{\substack{f=\mathrm{softmax}\circ \psi,\\
    \psi\in \mathcal{F}_{d,K}(\mathcal{W},\mathcal{D},\calB)} }\left |\calL(f)-\calL_n(f)  \right |  +\lambda\inf_{\phi \in \mathcal{F}_{d,K}(\mathcal{W},\mathcal{D},\calB)} \left \| \phi-\kappa \right \|_{L^2(\nu)},
\end{equation}
where $\nu$ is the the marginal probability measure of $X$. 
\end{lemma}

\begin{proof}
Let $\tilde f=\mathrm{softmax}\circ\tilde{\kappa}$ with $\tilde{\kappa} \in \mathcal{F}_{d,K}(\mathcal{W},\mathcal{D},\calB)$ satisfying $\|\tilde \kappa-\kappa\|_{L^2(\nu)}\le\|\phi-\kappa\|_{L^2(\nu)},\ \forall \phi\in\mathcal{F}_{d,K}(\mathcal{W},\mathcal{D},\calB)$, then 
\begin{equation}
\begin{aligned}
    \label{eq:error decomposition}
    &\calL(\hat{f}_n)-\calL(f_0)\\
    =&\calL(\hat{f}_n)-\calL_n(\hat{f}_n)+\calL_n(\hat{f}_n)-\calL_n(\tilde f)+\calL_n(\tilde f)-\calL(\tilde f)+\calL(\tilde f)-\calL(f_0)\\
    \le& \calL(\hat{f}_n)-\calL_n(\hat{f}_n)+\calL_n(\tilde f)-\calL(\tilde f)+\calL(\tilde f)-\calL(f_0)\\
    \le& 2\left[\sup_{\psi \in \mathcal{F}_{d,K}(\mathcal{W},\mathcal{D},\calB)} \left |\calL(\mathrm{softmax}\circ\psi)-\calL_n(\mathrm{softmax}\circ\psi)  \right | \right]+\left|\calL(\tilde f)-\calL(f_0)\right|.
    \end{aligned}
\end{equation}
Here, we utilize the fact that $\calL_n(\hat{f}_n)\le\calL_n( f),\ \forall f\in\mathcal{F}_{d,K}(\mathcal{W},\mathcal{D},\calB)$ in the first inequality of \eqref{eq:error decomposition}. Then
     \begin{equation}
    \begin{aligned}\label{C.3-1}
    |\calL(\tilde f)-\calL(f_0)|&=\left|\E_{X,Y}\left [  \ell(\tilde f(X),Y)\right ] -\E_{X,Y}\left [  \ell(f_0(X),Y)\right ]\right| \le \lambda \E_{X}\|\tilde \kappa(X)-\kappa(X)\|_2\\
    &\le\lambda\left \| \tilde \kappa-\kappa \right \|_{L^2(\nu )}=\lambda\inf_{\phi \in \mathcal{F}_{d,K}(\mathcal{W},\mathcal{D},\calB)}\left \| \phi-\kappa \right \|_{L^2(\nu)}.
    \end{aligned}
    \end{equation}   
 Combining \eqref{eq:error decomposition}, and \eqref{C.3-1}, we conclude the result.
\end{proof}
We denote $\sup_{f \in \mathcal{F}_{d,K}(\mathcal{W},\mathcal{D},\calB)} \left |\calL^\eta(f)-\calL^\eta_n(f)  \right |$ or $\sup_{f \in \mathcal{F}_{d,K}(\mathcal{W},\mathcal{D},\calB)} \left |\calL(f)-\calL_n(f)  \right |$ as the statistical error and $\inf_{\phi \in \mathcal{F}_{d,K}(\mathcal{W},\mathcal{D},\calB)}\left \| \phi-\kappa \right \|_{L^2(\nu)}$ as the approximation error. In the following subsections, we will explore both types of errors.
\subsection{Statistical error}\label{sub:statistical}
In practice, the samples in the dataset may not be independent. In this subsection, we focus on a dependent (mixing) sequence. To bound the statistical errors, we need the definition of independent block sequence (IB sequence).

\begin{definition}[Independent Block Sequence in \cite{Yu1994}]\label{def:ib1}
    Let $(a_n, \mu_n)$ be an integer pair such that \( n = a_n\mu_n \). We divide the strictly stationary sequence \( \{\hat x_i\}_{i=1}^{2n} \) into \( 2\mu_n \) blocks, each of length \( a_n \). 
For $1\le j\le \mu_n$, define
\begin{align*}
    G_j:=\left \{ i:2(j-1)a_n+1\le i\le (2j-1)a_n \right \},\ 
    \hat{x}_{1,a_n}=\bigcup^{\mu_n}_{j=1}\bigcup_{i\in G_j}\hat x_i.
\end{align*}
Let $\hat x_i'$ be an i.i.d. copy of $\hat x_i$ for $i=1,\cdots,2n$. That is,  $$\Xi_{1,a_n}=\bigcup^{\mu_n}_{j=1}\bigcup_{i\in G_j}\hat x'_i$$
is a sequence of identically distributed independent blocks such that  the sequence is independent of $\{\hat x_i\}^{2n}_{i=1}$ and each block has the same distribution as a block from the original sequence. We call $\Xi_{1,a_n}$ the independent block $a_n$-sequence.
\end{definition}

\begin{theorem}\label{theo:statistical}
   Let $\calF_{d,K}(\calW,\calD,\calB)$ represent a class of the norm-constrained neural network, characterized by an input dimension $d$, an output dimension of $K$, a width and depth denoted by $\calW$ and $\calD$, and a constraint $\calB>0$ satisfying $\| (W_{\calD},b_{\calD}) \|_{1,\infty} \prod_{i=0}^{\calD-1} \max\left\{ \| (W_{i}, b_{i}) \|_{1,\infty} , 1 \right\} \leq \calB$, where $W_i \in \mathbb{R}^{p_{i+1} \times p_i}$ is the weight matrix, and $b_i \in \mathbb{R}^{p_{i+1}}$ is the bias vector in the $i$-th layer for $i=0,1,\cdots,\calD$. Let $S=\left \{ z_i=(x_i,y_i) \right \}_{i=1}^{n}$ be a strictly stationary $\beta$-mixing sequence and $(a_n, \mu_n)$ be an integer pair such that \( n = a_n\mu_n \). Suppose the loss function we consider satisfying $|\ell(\mathrm{softmax}(a),q)-\ell(\mathrm{softmax}(b),q)|\le\lambda\|a-b\|_2,\forall a,b\in\R^K,q\in\calY$, where $\lambda>0$ is a given constant. Then we have  
\begin{equation*}
\mathbb{E}_S\left[\sup_{\substack{f=\mathrm{softmax}\circ \kappa,\\
    \kappa\in \mathcal{F}_{d,K}(\mathcal{W},\mathcal{D},\calB)} }\left | \calL(f)-\calL_n(f) \right | \right]
\le \frac{4\lambda \sqrt{K}\calB \sqrt{\calD+2+\log d}}{\sqrt{na_n}}+\frac{2 \lambda  \sqrt{K}n\beta_{a_n}}{a_n}.
\end{equation*}

\end{theorem}
\begin{proof}
    Let $\tilde{z_i}$ denote an independent copy of $z_i$, and denote $\tilde{S}=\{\tilde{z}_i\}^n_{i=1}$. Set $z_i=( x_i,y_i)$, $\tilde{z}_i=(\tilde x_i,\tilde y_i)$ and define $\ell_f(z_i)=\ell(f(x_i),y_i)$ and $\ell_f(\tilde z_i)=\ell(f(\tilde x_i),\tilde y_i)$. 
    Then we have
 \begin{align}\label{eq:theo5.4-1}
 &\mathbb{E}_S\left[\sup_{\substack{f=\mathrm{softmax}\circ \kappa,\\
    \kappa\in \mathcal{F}_{d,K}(\mathcal{W},\mathcal{D},\calB)} }\left | \calL(f)-\calL_n(f) \right | \right]
= \mathbb{E}_S\left[\sup_{\substack{f=\mathrm{softmax}\circ \kappa,\\
    \kappa\in \mathcal{F}_{d,K}(\mathcal{W},\mathcal{D},\calB)} }\left |\frac{1}{n} \sum_{i=1}^{n}\ell_f(z_i)-\mathbb{E}_{\tilde{S}}[\ell_f(\tilde{z_i})]  \right | \right]\nonumber\\
=&\mathbb{E}_S\left[\sup_{\substack{f=\mathrm{softmax}\circ \kappa,\\
    \kappa\in \mathcal{F}_{d,K}(\mathcal{W},\mathcal{D},\calB)} }\left |\frac{1}{n} \sum_{i=1}^{n}\mathbb{E}_{\tilde{S}}[\ell_f(z_i)]-\frac{1}{n}\sum_{i=1}^{n}\mathbb{E}_{\tilde S}[\ell_f(\tilde{z_i})]  \right | \right]\nonumber\\
\le&\mathbb{E}_{S,\tilde{S}}\left[\sup_{\substack{f=\mathrm{softmax}\circ \kappa,\\
    \kappa\in \mathcal{F}_{d,K}(\mathcal{W},\mathcal{D},\calB)} }\left |\frac{1}{n} \sum_{i=1}^{n}\ell_f(z_i)-\frac{1}{n}\sum_{i=1}^{n}\ell_f(\tilde{z_i}) \right | \right]\nonumber\\
=&2\mathbb{E}_{\hat S,\rho}\left[\sup_{\substack{f=\mathrm{softmax}\circ \kappa,\\
    \kappa\in \mathcal{F}_{d,K}(\mathcal{W},\mathcal{D},\calB)} }\left |\frac{1}{2n} \sum_{i=1}^{2n}\rho_i \ell_f(\hat z_i) \right | \right],
\end{align}
where $\hat S=S\cup\tilde{S}\triangleq\{\hat z_1,\cdots,\hat z_{2n}\}$, $\rho_i,i=1,\cdots,2n$ are i.i.d. Rademacher random variables that are independent of $\hat S$ and $\rho=\{\rho_i\}^{2n}_{i=1}$. 
Moreover, for any $f=\mathrm{softmax}\circ\kappa$ with $\kappa\in\calF_{d,K}(\calW,\calD,\calB)$, we have \begin{equation}\label{eq:theo5.4-1.2}
    \ell_f(\hat z_i)-\ell_f(\hat z'_i)\le\lambda\|\kappa(\hat x_i)-\kappa(\hat x'_i)\|_2\le \sqrt{K}\lambda|\kappa_{j_0}(\hat x_i)-\kappa_{j_0}(\hat x'_i)|,
\end{equation} with $j_0=\mathop{\arg\max}_{j\in\{1,\cdots,K\}}|\kappa^{(j)}_0(\hat x_i)-\kappa^{(j)}_0(\hat x'_i)|$.
Combining Equations \eqref{eq:theo5.4-1},  \eqref{eq:theo5.4-1.2} and Lemma \ref{Meir}, we have
\begin{equation}\label{eq:theo5.4-1.1}
    \mathbb{E}_S\left[\sup_{\substack{f=\mathrm{softmax}\circ \kappa,\\
    \kappa\in \mathcal{F}_{d,K}(\mathcal{W},\mathcal{D},\calB)} }\left | \calL(f)-\calL_n(f) \right | \right]\le 2\lambda \sqrt{K}\mathbb{E}_{ \hat S,\rho}\left[ 
        \sup_{\kappa \in \mathcal{F}_{d,1}(\calW,\calD,\calB)} 
        \left| \frac{1}{2n} \sum_{i=1}^{2n} \rho_i \kappa(\hat x_i)
        \right| 
        \right].
\end{equation} 

We partition the strictly stationary sequence \( \{\hat x_i\}_{i=1}^{2n} \) into \( 2\mu_n \) blocks, each of length \( a_n \). 
For $1\le j\le \mu_n$, define
\begin{align*}
    G_j:=\left \{ i:2(j-1)a_n+1\le i\le (2j-1)a_n \right \},\ 
    H_j:=\left \{ i:(2j-1)a_n+1\le i\le 2ja_n \right \}.
\end{align*}
In addition, we define 
\begin{align*}
    \hat x_{1,a_n}&=\bigcup^{\mu_n}_{j=1}\bigcup_{i\in G_j}\hat x_i,\quad \rho_{1,a_n}=\bigcup^{\mu_n}_{j=1}\bigcup_{i\in G_j}\rho_i,\quad y_{j,\kappa}(\hat x_{1,a_n},\rho_{1,a_n})=\sum_{i\in G_j}\rho_{i} \kappa(\hat x_i),\\
   \hat x_{2,a_n}&=\bigcup^{\mu_n}_{j=1}\bigcup_{i\in H_j}\hat x_i,\quad \rho_{2,a_n}=\bigcup^{\mu_n}_{j=1}\bigcup_{i\in H_j}\rho_i,\quad y_{j,\kappa}(\hat x_{2,a_n},\rho_{2,a_n})=\sum_{i\in H_j}\rho_{i} \kappa(\hat x_i).
\end{align*}
Let $\hat x_i'$ be an i.i.d. copy of $\hat x_i$ for $i=1,\cdots,2n$. Define $$\Xi_{1,a_n}=\bigcup^{\mu_n}_{j=1}\bigcup_{i\in G_j}\hat x'_i, \quad  \xi_{j,\kappa}(\Xi_{1,a_n},\rho_{1,a_n})=\sum_{i\in G_j}\rho_{i} \kappa(\hat x'_i).
$$
Then $\Xi_{1,a_n}$ is the IB sequence of $x_{1,a_n}$. We have

\begin{align}\label{eq:theo5.4-2}
        &\mathbb{E}_{ \hat S,\rho}\left[ 
        \sup_{\kappa \in  \calF_{d,1}(\calW,\calD,\calB)} 
        \left| \frac{1}{2n} \sum_{i=1}^{2n} \rho_i \kappa(\hat x_i)
        \right| 
        \right]\nonumber\\
 =&   \mathbb{E}_{ \hat S,\rho} 
\left( \sup_{\kappa \in  \calF_{d,1}(\calW,\calD,\calB)} \left| \frac{1}{2n} \sum_{j=1}^{\mu_n} y_{j,\kappa}(\hat x_{1,a_n},\rho_{1,a_n})+\frac{1}{2n} \sum_{j=1}^{\mu_n} y_{j,\kappa}(\hat x
_{2,a_n},\rho_{2,a_n}) \right|
\right)\nonumber\\
\le&  \mathbb{E}_{\hat x_{1,a_n}, \rho_{1,a_n}} 
\left( \sup_{\kappa \in  \calF_{d,1}(\calW,\calD,\calB)} \left| \frac{1}{2n} \sum_{j=1}^{\mu_n} y_{j,\kappa}( \hat x_{1,a_n},\rho_{1,a_n})\right|
\right)\nonumber\\
&\qquad+ \mathbb{E}_{\hat x_{2,a_n}, \rho_{2,a_n}} 
\left( \sup_{\kappa \in  \calF_{d,1}(\calW,\calD,\calB)} \left|  \frac{1}{2n} \sum_{j=1}^{\mu_n} y_{j,\kappa}(\hat x_{2,a_n},\rho_{2,a_n}) \right|
\right)\nonumber\\
=& 2\mathbb{E}_{\hat x_{1,a_n}, \rho_{1,a_n}} 
\left( \sup_{\kappa \in  \calF_{d,1}(\calW,\calD,\calB)} \left| \frac{1}{2n} \sum_{j=1}^{\mu_n} y_{j,\kappa}(\hat x_{1,a_n}, \rho_{1,a_n})\right|
\right)\nonumber\\
=&\mathbb{E}_{\hat x_{1,a_n}, \rho_{1,a_n}}
\left( \sup_{\kappa \in  \calF_{d,1}(\calW,\calD,\calB)}
\left| \frac{1}{\mu_n} \sum_{j=1}^{\mu_n} \frac{y_{j,\kappa}(\hat x_{1,a_n}, \rho_{1,a_n})}{a_n}\right|\right)\nonumber\\
\le&\mathbb{E}_{ \Xi_{1,a_n}, \rho_{1,a_n}}
\left( \sup_{\kappa \in  \calF_{d,1}(\calW,\calD,\calB)}
\left| \frac{1}{\mu_n} \sum_{j=1}^{\mu_n} \frac{\xi_{j,\kappa}(\Xi_{1,a_n}, \rho_{1,a_n})}{a_n}\right|\right)+\mu_n\beta_{a_n}\nonumber\\
    \le&\frac{2 \calB \sqrt{\calD+2+\log d}}{a_n\sqrt{\mu_n}}+ \mu_n\beta_{a_n}= \frac{2\calB \sqrt{\calD+2+\log d}}{\sqrt{na_n}}+\frac{ n\beta_{a_n}}{a_n},
\end{align} 
where the
second inequality to last is from Lemma \ref{Yu}, the last inequality is from Lemma \ref{Golowich} and the last equality is from $n=a_n\mu_n$. Combining \eqref{eq:theo5.4-1.1} and \eqref{eq:theo5.4-2}, we conclude the result.
\end{proof}

\begin{theorem}\label{theo:statistical1}
   Let $\calF_{d,K}(\calW,\calD,\calB)$ represent a class of the norm-constrained neural network, characterized by an input dimension $d$, an output dimension of $K$, a width and depth denoted by $\calW$ and $\calD$, and a constraint $\calB>0$ satisfying $\| (W_{\calD},b_{\calD}) \|_{1,\infty} \prod_{i=0}^{\calD-1} \max\left\{ \| (W_{i}, b_{i}) \|_{1,\infty} , 1 \right\} \leq \calB$, where $W_i \in \mathbb{R}^{p_{i+1} \times p_i}$ is the weight matrix, and $b_i \in \mathbb{R}^{p_{i+1}}$ is the bias vector in the $i$-th layer for $i=0,1,\cdots,\calD$. We consider the noisy dataset $S^\eta$ as a strictly stationary $\beta$-mixing sequence with the sample size $n$ satisfying \( n = a_n\mu_n \), where $(a_n, \mu_n)$ is an integer pair. Suppose the loss function we consider satisfying $|\ell(\mathrm{softmax}(a),q)-\ell(\mathrm{softmax}(b),q)|\le\lambda\|a-b\|_2,\forall a,b\in\R^K,q\in\calY$, where $\lambda>0$ is a given constant. Then we have  
\begin{equation*}
\mathbb{E}_{S^\eta}\left[\sup_{\substack{f=\mathrm{softmax}\circ \kappa,\\
    \kappa\in \mathcal{F}_{d,K}(\mathcal{W},\mathcal{D},\calB)} }\left | \calL^\eta(f)-\calL^\eta_n(f) \right | \right]
\le \frac{4\lambda \sqrt{K}\calB \sqrt{\calD+2+\log d}}{\sqrt{na_n}}+\frac{2 \lambda  \sqrt{K}n\beta_{a_n}}{a_n}.
\end{equation*}
\end{theorem}
\begin{proof}
    Let $\tilde{z_i}$ denote an independent copy of $z_i$, and denote $\tilde{S}=\{\tilde{z}_i\}^n_{i=1}$. Set $z_i=( x_i,y^\eta_i)$, $\tilde{z}_i=(\tilde x_i,\tilde y^\eta_i)$ and define $\ell_f(z_i)=\ell(f(x_i),y^\eta_i)$ and $\ell_f(\tilde z_i)=\ell(f(\tilde x_i),\tilde y^\eta_i)$. 
    Then we have
 \begin{align}\label{eq:theo5.5-1}
 &\mathbb{E}_S\left[\sup_{\substack{f=\mathrm{softmax}\circ \kappa,\\
    \kappa\in \mathcal{F}_{d,K}(\mathcal{W},\mathcal{D},\calB)} }\left | \calL(f)-\calL_n(f) \right | \right]
= \mathbb{E}_S\left[\sup_{\substack{f=\mathrm{softmax}\circ \kappa,\\
    \kappa\in \mathcal{F}_{d,K}(\mathcal{W},\mathcal{D},\calB)} }\left |\frac{1}{n} \sum_{i=1}^{n}\ell_f(z_i)-\mathbb{E}_{\tilde{S}}[\ell_f(\tilde{z_i})]  \right | \right]\nonumber\\
=&\mathbb{E}_S\left[\sup_{\substack{f=\mathrm{softmax}\circ \kappa,\\
    \kappa\in \mathcal{F}_{d,K}(\mathcal{W},\mathcal{D},\calB)} }\left |\frac{1}{n} \sum_{i=1}^{n}\mathbb{E}_{\tilde{S}}[\ell_f(z_i)]-\frac{1}{n}\sum_{i=1}^{n}\mathbb{E}_{\tilde S}[\ell_f(\tilde{z_i})]  \right | \right]\nonumber\\
\le&\mathbb{E}_{S,\tilde{S}}\left[\sup_{\substack{f=\mathrm{softmax}\circ \kappa,\\
    \kappa\in \mathcal{F}_{d,K}(\mathcal{W},\mathcal{D},\calB)} }\left |\frac{1}{n} \sum_{i=1}^{n}\ell_f(z_i)-\frac{1}{n}\sum_{i=1}^{n}\ell_f(\tilde{z_i}) \right | \right]\nonumber\\
=&2\mathbb{E}_{\hat S,\rho}\left[\sup_{\substack{f=\mathrm{softmax}\circ \kappa,\\
    \kappa\in \mathcal{F}_{d,K}(\mathcal{W},\mathcal{D},\calB)} }\left |\frac{1}{2n} \sum_{i=1}^{2n}\rho_i \ell_f(\hat z_i) \right | \right],
\end{align}
where $\hat S=S\cup\tilde{S}\triangleq\{\hat z_1,\cdots,\hat z_{2n}\}$, $\rho_i,i=1,\cdots,2n$ are i.i.d. Rademacher random variables that are independent of $\hat S$ and $\rho=\{\rho_i\}^{2n}_{i=1}$. 
Moreover, for any $f=\mathrm{softmax}\circ\kappa$ with $\kappa\in\calF_{d,K}(\calW,\calD,\calB)$, we have \begin{equation}\label{eq:theo5.5-1.2}
    \ell_f(\hat z_i)-\ell_f(\hat z'_i)\le\lambda\|\kappa(\hat x_i)-\kappa(\hat x'_i)\|_2\le \sqrt{K}\lambda|\kappa_{j_0}(\hat x_i)-\kappa_{j_0}(\hat x'_i)|,
\end{equation} with $j_0=\mathop{\arg\max}_{j\in\{1,\cdots,K\}}|\kappa^{(j)}_0(\hat x_i)-\kappa^{(j)}_0(\hat x'_i)|$.
Combining Equations \eqref{eq:theo5.5-1},  \eqref{eq:theo5.5-1.2} and Lemma \ref{Meir}, we have
\begin{equation}\label{eq:theo5.5-1.1}
    \mathbb{E}_S\left[\sup_{\substack{f=\mathrm{softmax}\circ \kappa,\\
    \kappa\in \mathcal{F}_{d,K}(\mathcal{W},\mathcal{D},\calB)} }\left | \calL(f)-\calL_n(f) \right | \right]\le 2\lambda \sqrt{K}\mathbb{E}_{ \hat S,\rho}\left[ 
        \sup_{\kappa \in \mathcal{F}_{d,1}(\calW,\calD,\calB)} 
        \left| \frac{1}{2n} \sum_{i=1}^{2n} \rho_i \kappa(\hat x_i)
        \right| 
        \right].
\end{equation} 

We partition the strictly stationary sequence \( \{\hat x_i\}_{i=1}^{2n} \) into \( 2\mu_n \) blocks, each of length \( a_n \). 
For $1\le j\le \mu_n$, define
\begin{align*}
    G_j:=\left \{ i:2(j-1)a_n+1\le i\le (2j-1)a_n \right \},\ 
    H_j:=\left \{ i:(2j-1)a_n+1\le i\le 2ja_n \right \}.
\end{align*}
In addition, we define 
\begin{align*}
    \hat x_{1,a_n}&=\bigcup^{\mu_n}_{j=1}\bigcup_{i\in G_j}\hat x_i,\quad \rho_{1,a_n}=\bigcup^{\mu_n}_{j=1}\bigcup_{i\in G_j}\rho_i,\quad y_{j,\kappa}(\hat x_{1,a_n},\rho_{1,a_n})=\sum_{i\in G_j}\rho_{i} \kappa(\hat x_i),\\
   \hat x_{2,a_n}&=\bigcup^{\mu_n}_{j=1}\bigcup_{i\in H_j}\hat x_i,\quad \rho_{2,a_n}=\bigcup^{\mu_n}_{j=1}\bigcup_{i\in H_j}\rho_i,\quad y_{j,\kappa}(\hat x_{2,a_n},\rho_{2,a_n})=\sum_{i\in H_j}\rho_{i} \kappa(\hat x_i).
\end{align*}
Let $\hat x_i'$ be an i.i.d. copy of $\hat x_i$ for $i=1,\cdots,2n$. Define $$\Xi_{1,a_n}=\bigcup^{\mu_n}_{j=1}\bigcup_{i\in G_j}\hat x'_i, \quad  \xi_{j,\kappa}(\Xi_{1,a_n},\rho_{1,a_n})=\sum_{i\in G_j}\rho_{i} \kappa(\hat x'_i).
$$
Then $\Xi_{1,a_n}$ is the IB sequence of $x_{1,a_n}$. We have

\begin{align}\label{eq:theo5.5-2}
        &\mathbb{E}_{ \hat S,\rho}\left[ 
        \sup_{\kappa \in  \calF_{d,1}(\calW,\calD,\calB)} 
        \left| \frac{1}{2n} \sum_{i=1}^{2n} \rho_i \kappa(\hat x_i)
        \right| 
        \right]\nonumber\\
 =&   \mathbb{E}_{ \hat S,\rho} 
\left( \sup_{\kappa \in  \calF_{d,1}(\calW,\calD,\calB)} \left| \frac{1}{2n} \sum_{j=1}^{\mu_n} y_{j,\kappa}(\hat x_{1,a_n},\rho_{1,a_n})+\frac{1}{2n} \sum_{j=1}^{\mu_n} y_{j,\kappa}(\hat x
_{2,a_n},\rho_{2,a_n}) \right|
\right)\nonumber\\
\le&  \mathbb{E}_{\hat x_{1,a_n}, \rho_{1,a_n}} 
\left( \sup_{\kappa \in  \calF_{d,1}(\calW,\calD,\calB)} \left| \frac{1}{2n} \sum_{j=1}^{\mu_n} y_{j,\kappa}( \hat x_{1,a_n},\rho_{1,a_n})\right|
\right)\nonumber\\
&\qquad+ \mathbb{E}_{\hat x_{2,a_n}, \rho_{2,a_n}} 
\left( \sup_{\kappa \in  \calF_{d,1}(\calW,\calD,\calB)} \left|  \frac{1}{2n} \sum_{j=1}^{\mu_n} y_{j,\kappa}(\hat x_{2,a_n},\rho_{2,a_n}) \right|
\right)\nonumber\\
=& 2\mathbb{E}_{\hat x_{1,a_n}, \rho_{1,a_n}} 
\left( \sup_{\kappa \in  \calF_{d,1}(\calW,\calD,\calB)} \left| \frac{1}{2n} \sum_{j=1}^{\mu_n} y_{j,\kappa}(\hat x_{1,a_n}, \rho_{1,a_n})\right|
\right)\nonumber\\
=&\mathbb{E}_{\hat x_{1,a_n}, \rho_{1,a_n}}
\left( \sup_{\kappa \in  \calF_{d,1}(\calW,\calD,\calB)}
\left| \frac{1}{\mu_n} \sum_{j=1}^{\mu_n} \frac{y_{j,\kappa}(\hat x_{1,a_n}, \rho_{1,a_n})}{a_n}\right|\right)\nonumber\\
\le&\mathbb{E}_{ \Xi_{1,a_n}, \rho_{1,a_n}}
\left( \sup_{\kappa \in  \calF_{d,1}(\calW,\calD,\calB)}
\left| \frac{1}{\mu_n} \sum_{j=1}^{\mu_n} \frac{\xi_{j,\kappa}(\Xi_{1,a_n}, \rho_{1,a_n})}{a_n}\right|\right)+\mu_n\beta_{a_n}\nonumber\\
    \le&\frac{2 \calB \sqrt{\calD+2+\log d}}{a_n\sqrt{\mu_n}}+ \mu_n\beta_{a_n}= \frac{2\calB \sqrt{\calD+2+\log d}}{\sqrt{na_n}}+\frac{ n\beta_{a_n}}{a_n},
\end{align} 
where the
second inequality to last is from Lemma \ref{Yu}, the last inequality is from Lemma \ref{Golowich} and the last equality is from $n=a_n\mu_n$. Combining \eqref{eq:theo5.5-1.1} and \eqref{eq:theo5.5-2}, we conclude the result.
\end{proof}

\subsection{Approximate error}\label{sub:approx}
 Next, we estimate the boundary of the approximation error, we have the following theorem.

\begin{theorem}\label{theo:approximation}
    For any $\kappa\in \calC_{\tau,d,K,V}$, there exists $\phi\in \calF_{d,K}(\calW,\calD,\calB)$ with $\calW\gtrsim \calB^{d/(d+1)} \log\calB$ and $\calD\gtrsim \log\calB$
    such that
    \begin{equation}
        \left \|\kappa-\phi \right \| _{L^2(\nu)}\lesssim \sqrt{K} \calB^{-\tau/(d+1)}.
    \end{equation}
\end{theorem}
 \begin{proof}
     Let  
     \begin{equation*}
         \psi(t)=\sigma(1-\left | t \right | )=\sigma(1-\sigma(t)-\sigma(-t))\in [0,1],\ \forall t\in\R,
     \end{equation*}
     then $\psi\in \calF_{1,1}(2,2,3)$ and is supported  on $[-1,1]$. 
For any $n=[n_1\cdots n_d]^\top \in \left \{ 0,1,\cdots ,N \right \} ^d$, let
\begin{equation*}
    \psi_n(x)=\prod_{i=1}^{d}   \psi(Nx_i-n_i),\quad \forall x=\begin{bmatrix}
        x_1&\cdots&x_d
    \end{bmatrix}^\top\in [0,1]^d,
\end{equation*}
$\psi_n$ is supported on $\left\{x\in \mathbb{R}^d: \left \| x-\frac{n}{N}  \right \| _{\infty }\le \frac{1}{N} \right\}$. Moreover,
\begin{equation*}
    \sum_{ n\in\left \{0,1,\cdots,N \right \}^d  }\psi_n(x)=\prod_{i=1}^{d}  \sum_{n_i=0}^{N} \psi(Nx_i-n_i)\equiv 1,\quad x\in [0,1]^d.
\end{equation*}
Note that $x_i-n_i/N=\sigma(x_i-n_i/N)-\sigma(-x_i+n_i/N)\in \calF_{1,1}(2,1,2)$, $\psi(Nx_i-n_i) \in \calF_{1,1}(2,2,6N)$. Let $s=[s_1\  s_2\cdots s_d]^\top$ is a multi-index satisfying $\left\|s\right\|_1 \le r$, according to Lemmas \ref{lem:relation} and \ref{lem:high dimension}, there exists a neural network $\phi_{n,s}\in \calF_{d,1}((6k+3)(d+r),(k+1)(d+r)+2,6N(d+r)^7(2(d+r))^{k+1})$ such that 
\begin{equation}\label{eq:approx_nn}
    \left|\psi_n(x)\left(x-\frac{n}{N}\right)^s-\phi_{n,s}(x)\right|\le3(d+r)2^{-2k},
\end{equation}
and $\phi_{n,s}(x)=0$ if $\psi_n(x)\left(x-\frac{n}{N}\right)^s=0$, where $\left(x-\frac{n}{N}\right)^s=\Pi^d_{i=1}(x_i-n_i/N)^{s_i}$.

According to  Lemma \ref{PetersenA.8}, for any $\kappa=[\kappa_1\ \cdots\ \kappa_K]^\top\in \calC_{\tau,d,K,V}$, any $x\in[0,1]^d$ and any $n\in\{0,\cdots,N\}^d$, there exists a polynomial $\sum_{\left \| s \right \|_1 \le r } c^{(i)}_{n,s}(x-\frac{n}{N})^s$, we have
\begin{equation}\label{eq:approx_fi}
    \left | \kappa_i(x) -\sum_{\left \| s \right \|_1 \le r } c^{(i)}_{n,s}\left(x-\frac{n}{N}\right)^s \right |\le 2^dV|\max_{j=1,\cdots,d} \left|x_j-\frac{n_j}{N}\right|^\tau,
\end{equation}
where $c^{(i)}_{n,s}=\partial ^\alpha \kappa_i(\frac{n}{N})/s!$. 
Define $p(x)=[p_1(x)\cdots p_K(x)]^\top$ and $\phi(x)=[\phi_1(x)\cdots \phi_K(x)]^\top$ with 
\begin{equation*}
    p_i(x)= \sum_{ n\in\left \{{0,1,\cdots,N} \right \}^d  } \sum_{\left \| s \right \|_1 \le r } c_{n,s}^{(i)}\psi_n(x)\left(x-\frac{n}{N}\right)^s,\ \phi_i(x)=\sum_{ n\in\left \{{0,1,\cdots,N} \right \}^d  }\sum_{\left \| s \right \|_1 \le r } c_{n,s}^{(i)} \phi_{n,s}(x).
\end{equation*}
Combined with  Lemma \ref{lem:relation} and \begin{equation}\label{eq:num}
    \sum_{\left \| s \right \|_1\le r }1=\sum_{i=0}^{r} \sum_{\left \| s \right \|_i=i }1\le \sum_{i=0}^{r} d^i \le (r+1)d^r, 
\end{equation} we know that $\phi\in \calF_{d,K}(K\gamma(N+1)^d(6k+3)(d+r),(k+1)(d+r)+2,6\gamma(N+1)^dN(d+r)^7(2(d+r))^{k+1})$ with $\gamma=\sum_{\left \| s \right \|_1\le r }1\le(r+1)d^r$. Then
\begin{align*}
    \left \| \kappa(x)-p(x) \right \|_2 &=\sqrt{\sum^K_{i=1}\left | \sum_{n}\psi_n(x)\kappa_i(x) -\sum_{n}\psi_n(x)\sum_{\left \| s \right \|_1 \le r } c_{n,s}^{(i)}\left(x-\frac{n}{N}\right)^s \right |^2} \\
    &\le \sqrt{K}\max_{i=1,\cdots,K}\left(\sum_{n}\psi_n(x)\left | \kappa_i(x) -\sum_{\left \| s \right \|_1 \le r } c_{n,s}^{(i)}\left(x-\frac{n}{N}\right)^s \right |\right) \\
    &\le \sqrt{K}\max_{i=1,\cdots,K}\sum_{ n:\left \| x-\frac{n}{N}  \right \|_{\infty}\le \frac{1}{N} }\left | \kappa_i(x) -\sum_{\left \| s \right \|_1 \le r } c_{n,s}^{(i)}\left(x-\frac{n}{N}\right)^s \right | \\
    &\le \sum_{ n:\left \| x-\frac{n}{N}  \right \|_{\infty}\le \frac{1}{N} }d^{r}V\sqrt{K} \left \| x-\frac{n}{N}  \right \| _{\infty }^{\tau}\le2^d\sqrt{K} d^rV N^{-\tau},
\end{align*}
where the third inequality is from Equation \eqref{eq:approx_fi}.
Similarly, for any $x\in[0,1]^d$, the approximation error is
\begin{align*}
   \left \| p(x)-\phi (x) \right \|_2 
   &= \sqrt{\sum^K_{i=1}\left | \sum_{n}\sum_{\left \| s \right \|_1 \le r } c^{(i)}_{n,s} \psi_n(x)\left(x-\frac{n}{N}\right)^s -\sum_{n}\sum_{\left \| s \right \|_1 \le r } c^{(i)}_{n,s} \phi_{n,s}(x)\right |^2} \\ 
   &\le\sqrt{K} \max_{i=1,\cdots,K}\sum_{n} \sum_{\left | s \right \|_1 \le r }\left |  c^{(i)}_{n,s} \right |  \left |
    \psi_n(x)(x-\frac{n}{N})^s-\phi_{n,s}(x) \right |  \\
   &\le V\sqrt{K} \max_{i=1,\cdots,K}\sum_{n} \sum_{\left | s \right \|_1 \le r }\left |
    \psi_n(x)(x-\frac{n}{N})^s-\phi_{n,s}(x) \right | \\
    &\le 3V\sqrt{K}\cdot2^d (r+1)(d+r)d^r 2^{-2k}.
\end{align*}
Here we use the facts that $\left | c^{(i)}_{n,s} \right | =\left | \partial ^s\kappa_i(\frac{n}{N} ) /s!\right | \le V$ (Definition \ref{def:smooth_maps}), Equations \eqref{eq:approx_nn} and \eqref{eq:num}.  In summary, the total approximation error is 
\begin{equation*}
    \left \| \kappa(x)-\phi(x) \right \|_2 \le \left \| \kappa(x)-p(x) \right \|_2 +\left \| p(x)-\phi(x) \right \|_2 \le V\sqrt{K}2^d d^r(N^{-\tau}+3(d+r)(r+1)2^{-2k}).
\end{equation*}
We choose $N=\left \lceil 2^{2k/\tau}  \right \rceil $, then $\phi \in \calF_{d,K}(\calW,\calD,\calB)$ with
\begin{align*}
        K(N+1)^d(6k+3)(d+r)\le\calW&\le K(r+1)d^r(N+1)^d(6k+3)(d+r) \asymp 2^{{2dk}/\tau} k,\\
        \calD&=(k+1)(d+r)+2,\\
6(N+1)^dN(d+r)^7(2(d+r))^{k+1}\le\calB&\le6(r+1)d^r(N+1)^dN(d+r)^7(2(d+r))^{k+1} \\
&\asymp 2^{{2(d+1)k}/\tau},
    \end{align*}
such that $ \left \| \kappa(x)-\phi(x) \right \|_2  \le \sqrt{K}2^d d^r(N^{-\tau}+3(d+r)(r+1)2^{-2k})\lesssim\sqrt{K} 2^{-2k}$. Then, $k\asymp \log \calB$, $\calW\asymp \calB^{d/(d+1)} \log\calB$, $\calD\asymp \log\calB$. We have the following approximation bound
\begin{equation*}
    \left \| \kappa(x)-\phi(x) \right \|_{L^2(\nu)}=\sqrt{\int_{[0,1]^d}\left \| \kappa(x)-\phi(x) \right \|^2_2 d\nu(x)} \lesssim\sqrt{K} 2^{-2k} \lesssim \sqrt{K} \calB^{-\tau/(d+1)}.
\end{equation*}
The above approximation bound holds for $\calW\gtrsim \calB^{d/(d+1)} \log\calB$ and $\calD\gtrsim \log\calB$.
 \end{proof}
\begin{proof}[Proof of Theorem \ref{theo:excess_risk}.]
    We have 
\begin{align*}
    &\E_{S^\eta}[\calL^\eta(\hat{f^\eta_n})-\calL^\eta(f_0)]
    \le2 \E_{S^\eta }
\sup_{\substack{f=\mathrm{softmax}\circ \psi,\\
    \psi\in \mathcal{F}_{d,K}(\mathcal{W},\mathcal{D},\calB)} } \left |\calL^\eta(f)-\calL^\eta_n(f)  \right |  +\lambda\inf_{\phi \in \mathcal{F}_{d,K}(\mathcal{W},\mathcal{D},\calB)}\left \| \phi-\kappa \right \|_{L^2(\nu)}\\
    =&2 \E_{S^\eta|S }\left[\E_S
\sup_{\substack{f=\mathrm{softmax}\circ \psi,\\
    \psi\in \mathcal{F}_{d,K}(\mathcal{W},\mathcal{D},\calB)} } \left |\calL^\eta(f)-\calL^\eta_n(f)  \right | \right] +\lambda\inf_{\phi \in \mathcal{F}_{d,K}(\mathcal{W},\mathcal{D},\calB)}\left \| \phi-\kappa \right \|_{L^2(\nu)}\\
     \lesssim& \frac{8\lambda \sqrt{K}\calB \sqrt{\calD+2+\log d}}{\sqrt{na_n}}+\frac{4 \lambda  \sqrt{K}n\beta_{a_n}}{a_n}+\lambda\sqrt{K} \calB^{-\tau/(d+1)},
\end{align*}
where the first inequality is from Lemma \ref{lem:loss1},  the second inequality is from Theorems \ref{theo:statistical1} and \ref{theo:approximation}. Similarly, we have
\begin{align*}
    &\E_S[\calL(\hat{f_n})-\calL(f_0)]\\
    \le&
2 \E_S\left[\sup_{\substack{f=\mathrm{softmax}\circ \psi,\\
    \psi\in \mathcal{F}_{d,K}(\mathcal{W},\mathcal{D},\calB)} }\left |\calL(f)-\calL_n(f)  \right | \right] +\lambda\inf_{\phi \in \mathcal{F}_{d,K}(\mathcal{W},\mathcal{D},\calB)} \left \| \phi-\kappa \right \|_{L^2(\nu)}\\
    \lesssim& \frac{8\lambda \sqrt{K}\calB \sqrt{\calD+2+\log d}}{\sqrt{na_n}}+\frac{4 \lambda  \sqrt{K}n\beta_{a_n}}{a_n}+\lambda\sqrt{K} \calB^{-\tau/(d+1)},
\end{align*}
where the first inequality is from Lemma \ref{lem:loss},  the second inequality is from Theorems \ref{theo:statistical} and \ref{theo:approximation}.
\end{proof}

\section{The Curse of Dimensionality}\label{sec:curse}
In Subsection \ref{sub:approx}, the upper bound for the approximation error is given by $\sqrt{K} \calB^{-\tau/(d+1)}$. However, as $d$ approaches infinity, this bound exhibits explosive growth, resulting in the curse of dimensionality. 
In this section, we mitigate this issue based on the following assumption.
\begin{assumption}\label{assump:Low-dimensional-manifold-hypothesis}
    The feature space is supported on a compact $s$-dimensional Riemannian manifold $ \mathcal{A}$ embedded in $\R^d$, which can be decomposed into a union of $C$ open sets in $\R^d$  and $ C \ll d $. That is, $ \mathcal{A}=\bigcup^C_{j=1}U_j$. For $ j \in \{1, 2, \dots, C\} $, there exists a linear and invertible map $ \zeta_j,j=1,\cdots,C$ with $ s < d $ such that $ \zeta_j(x) \in  [0,1]^s,\forall x\in U_j $. 
\end{assumption}

Assumption \ref{assump:Low-dimensional-manifold-hypothesis} is well-founded. There is a vast amount of data generated and collected in our daily lives. Despite their diverse sources and natures, these seemingly unstructured data often possess inherent structures. One commonly used structure is that high-dimensional data may reside in a low-dimensional subspace rather than being arbitrarily distributed throughout the entire space. For instance, handwritings of the same digit with varying rotations, translations, or thickness are approximately within a low-dimensional subspace. Other examples include human face image data, where the face images of each individual under different illuminations form a low-dimensional subspace, and a video sequence with multiple moving objects, where each moving object across different frames belongs to a low-dimensional subspace. 

In addition, the assumption about the linear and invertible relationship is also reasonable. In the realm of compressive sensing, accurate recovery of the true sparse signal is possible when the sensing matrix fulfills the Restricted Isometry Property (RIP) \cite{candes2005decoding,candes2008restricted}, or Block Restricted Isometry Property (BRIP) \cite{kamali2013block,he2025group} condition. Furthermore, the theory of differential manifold \cite{brand2002charting,jiao2023} guarantees the existence of local linearity and reversibility. 

In the following, we will give a new estimation of the error bound of the approximation error based on Assumption \ref{assump:Low-dimensional-manifold-hypothesis}. 
\begin{theorem}\label{theo:curse}
    Assuming that Assumption \ref{assump:Low-dimensional-manifold-hypothesis} holds, for any map $\kappa\in\calC_{\tau,d,K,V}$ on a compact $s$-dimensional Riemannian manifold $ \mathcal{A}$ embedded in $\R^d$, there exists $\calB>0$ with $\calW'\gtrsim C\calB^{s/(s+1)} \log\calB$ and $\calD'\gtrsim \log\calB$, and $\phi \in \calF_{d,K} (\calW',\calD',{\calB}')$ with ${\calB}'=C\calB$,  we have 
\begin{equation}
    \left \| \phi-\kappa \right \|_{L^2{(\nu)}} 
    \lesssim \sqrt{K} \calB^{-\tau/(s+1)}.
\end{equation}
\end{theorem}
\begin{proof}
Let $\rho_j,j=1,\cdots,C$ be elements in $C^\infty$ partition of unity on $\calA$ supported on $U_j$ and $\sum^C_{j=1}\rho_j=1$ (Theorem 13.7 in \cite{tu2011manifolds}). Then $\kappa$ can be decomposed as $\kappa=\sum^C_{j=1}\kappa\rho_j=\sum^C_{j=1}\kappa_j$, where $\kappa_j\in\calC_{\tau,s,K,V}$ is supported on $U_j$. Set $\tilde{\kappa}_j=\kappa_j\circ\zeta^{-1}_j$, which is supported on $\zeta_j(U_j)\in[0,1]^s$. From Theorem \ref{theo:approximation}, there exists a positive constant $\calB>0$ and a map $\tilde{\phi}_j\in \calF_{s,K} (\calW,\calD,\calB)$ for $\tilde{\kappa}_j$,  with $\calW\gtrsim \calB^{s/(s+1)} \log\calB$ and $\calD\gtrsim \log\calB$ such that
\begin{equation*}
    \left \| \tilde{\phi_j}-\tilde{\kappa}_j \right \| _{L^2(\nu)}\lesssim \sqrt{K} \calB^{-\tau/(s+1)},j=1,\cdots,C.
 \end{equation*}
Next we define 
\begin{equation*}
    \phi(x)=\sum_{j=1}^{C} \tilde{\phi}_j(\zeta_j(x))=\sum_{j=1}^{C}\tilde{\phi}_j\circ \zeta_j(x),
\end{equation*}
then according to (3) in Lemma \ref{lem:relation} we know that $\phi$ is also a ReLU neural network with the same depth as $\tilde{\phi}_j$,  but its width is $C$ times that of $\tilde{\phi}_j$. That is $\phi \in \calF_{d,K} (\calW',\calD',{\calB}')$ with $\calW'\gtrsim C\calB^{s/(s+1)} \log\calB$, $\calD'\gtrsim \log\calB$ and $\calB'=C\calB$. Hence
\begin{align*}
    &\left \| \phi-\kappa \right \|_{L^2{(\nu)}} 
    \le\sum_{j=1}^{C}\left \| \tilde{\phi}_j(\zeta_j(x))-\tilde{\kappa}_j(\zeta_j(x))\right \|_{L^2{(\nu)}}
    \overset{\tilde{x}=\zeta_j(x)}{\le} \sum_{j=1}^{C}\sqrt{|Q_j|}\left \| \tilde{\phi}_j(\tilde{x})-\tilde{\kappa}_j(\tilde x)\right \|_{L^2{(\nu)}}\\
    \lesssim& \sqrt{K} \calB^{-\tau/(s+1)},
\end{align*}
where $Q_j>0$ are the Jacobi determinants of $\tilde{x}=\zeta_j(x),j=1,\cdots,C$. 
\end{proof}
To prove Theorem \ref{theo:curse}, we construct a specific mapping $\tilde{\kappa}_j=\kappa_j\circ\zeta^{-1}_j,j=1,\cdots,C$ from $\zeta_j(U_j)$ to $\R^K$. Our proof mainly includes the following steps: 

We first estimate the upper bound of the approximation error for each map $\tilde{\kappa}_j,j=1,\cdots,C$.  
According to Theorem \ref{theo:approximation}, for each function, we can derive an upper bound for its approximation error. This bound depends on $s$, the intrinsic dimensionality of the domain, rather than $d$.

Assumption \ref{assump:Low-dimensional-manifold-hypothesis} states that $\zeta_j$ for all $j$ are linear maps. These functions establish a one-to-one correspondence between the Riemannian manifold and a transformed space. This correspondence is crucial because it allows us to map the simpler, lower-dimensional space back to the Riemannian manifold.

By combining the results from the individual bounds on each function from $\zeta_j(U_j)$ to $\R^K$ and leveraging the invertible relationship, we demonstrate that the overall approximation error for $\kappa$ is controlled by a function of $s$. This is significant because it shows that, despite potentially high ambient dimensionality $d$, the complexity of approximation is determined by the lower intrinsic dimensionality $s$.
\section{Conclusion}\label{sec:conclusion}
Labeling large-scale datasets is a costly and error-prone process, and even high-quality datasets are inevitably subject to mislabeling. Training deep neural networks on cheaply obtained noisy labeled datasets is a common task. In this paper, we focus on the error bounds of excess risk for classification problems with noisy labels within deep learning frameworks.

We estimate the error bound of the excess risk, expressed as a sum of statistical error and approximation error. We assess the statistical error on a dependent (mixing) sequence, bounding it with the help of the corresponding independent block sequence. For the approximation error, we establish these theoretical results to the vector-valued setting. Finally, we focus on the curse of dimensionality based on the low-dimensional manifold assumption. This work does not provide numerical validation herein, as it has been extensively covered in some existing works \cite{ghazi2021,liu2026blockrr}.
\section*{Acknowledgment}
Part of work of H. Liu was done during a visit at The Hong Kong University of Science and Technology. The work of H. Liu was supported in part by Interdisciplinary Research Program of HUST 2024JCYJ005, National Key Research and Development Program of China 2023YFC3804500. The work of Dr. Can Yang was supported in part by Hong Kong Research Grants Council Grants 16308120, 16307221, 16307322, 16302823 and 16309424; The Hong Kong University of Science and Technology Startup Grants R9405 and Z0428 from the Big Data Institute. 
\bibliographystyle{elsarticle-num}
\bibliography{mybib} 
\appendix
\section{Useful Lemmas}
\begin{lemma}[Lemma from Subsection 4.1 in \cite{Yu1994}]\label{Yu}
    Let $x_{1,a_n}$ and $\Xi_{1,a_n}$ be the mixing and independent block sequences defined in Definition \ref{def:ib1}, whose distributions are  $\mathcal{T}$ and $\tilde{\mathcal{T}}$, respectively. For any measurable function $g$ on $\R^{\mu_n a_n}$ with bound $\tilde{M}$,
    \begin{equation*}
        \left | \mathcal{T}g(x_{1,a_n})- \tilde{\mathcal{T}}g(\Xi_{1,a_n})\right | \le \tilde{M}(\mu_n-1)\beta_{a_n},
    \end{equation*}
    where $(a_n,\mu_n)$ is an integer pair with $n=2a_n\mu_n$, and $\beta_{a_n}$ is defined in \eqref{eq:beta-mixing}.
\end{lemma}

\begin{lemma}[Lemma 5 in \cite{Meir2003}]\label{Meir}
    Let $\left\{g_i(\theta)\right\}$ and $\left\{h_i(\theta)\right\}$ be sets of functions defined for all $\theta$ in some domain $\Theta $. If $\left | g_i(\theta)- g_i(\theta')\right | \le \left | h_i(\theta)- h_i(\theta')\right |$ holds for all $i,\theta,\theta'$, then for any function $c(x,\theta)$ and probability
    distribution of x is over $X$, we have
 \begin{equation*}
     \mathbb{E}_{\theta}\mathbb{E}_{X}\sup_{\theta \in \Theta}\left\{c(X,\theta)+\sum_{i=1}^{n}\sigma_{i}g_{i}(\theta)\right\} \le \mathbb{E}_{\theta}\mathbb{E}_{X}\sup_{\theta \in \Theta}\left\{c(X,\theta)+\sum_{i=1}^{n}\sigma_{i}h_{i}(\theta)\right\}.
 \end{equation*}
\end{lemma}

\begin{lemma}[Theorem 2 in  \cite{Golowich2018}]\label{Golowich}
    Let $\calF_{d,1}(\calW,\calD)$ represent a class of neural networks, characterized by an input dimension $d$, an output dimension of $1$, and a width and depth denoted by $\calW$ and $\calD$, respectively. Define $W_j$ as the weight matrix of the $j$-th layer, satisfying $\|W_j\|_{1,\infty}\le M( j)$ for all  $j=1,\cdots,\calD+1$, where $\|W_j\|_{1,\infty}$
  denotes the infinity norm of the 1-norms of the columns of $W_j$. Let $\hat S=\{x_i\}^n_{i=1}$ represent
$n$ input features such that $\|x_i\|_2\le B,\forall i$ and $\sigma_i,i=1,\cdots,n$ are i.i.d. Rademacher random variables that are independent of $\hat S$. 
    Then 
     \begin{equation*}
     \begin{split}
         \mathbb{E}_{\{x_i, \sigma_i\}_{i=1}^n}
    \left[ 
    \sup_{f \in \mathcal{F}_{d,1}(\calW,\calD,\calB)} 
    \left| \frac{1}{n} \sum_{i=1}^n \sigma_i f(x_i)
    \right| 
    \right]\leq&\dfrac{2}{ n}\prod\limits_{j=1}^{\calD+1}M(j)\cdot\sqrt{\calD+2+\log d}\cdot\sqrt{\max_{i\in\{1,\cdots,n\}}\|x_i\|^2_2}\\
\le& \frac{2 B\sqrt{\calD+2+\log d}\cdot\prod_{j=1}^{\calD+1}M(j)}{\sqrt{ n}}.
    \end{split}
    \end{equation*}
\end{lemma}
\begin{lemma}
[Lemma 25 in \cite{feng2023}]\label{lem:relation}
    
    Let $\mathcal{F}_{d,K}(\calW,\calD,\calB)$ denote the class of functions generated by a ReLU neural network with input dimension $d$, output dimension $K$, width 
 $\calW$ and depth $\calD$, with weights bounded by $\calB$. Suppose $g_1 \in \mathcal{F}_{d_1,K_1}(\calW_1,\calD_1,\calB_1)$, $g_2 \in \mathcal{F}_{d_2,K_2}(\calW_2,\calD_2,\calB_2)$, then the following properties hold:\\
    (1) If $d_1=d_2$, $K_1=K_2$, $\calW_1 \le \calW_2$, $\calD_1 \le \calD_2$, $\calB_1\le\calB_2$, then ${\calF}_{d_1,K_1}(\calW_1,\calD_1,\calB_1)\subseteq {\calF}_{d_2,K_2}(\calW_2,\calD_2,\calB_2)$.\\
    (2) If $K_1=d_2$, then $g_2 \circ g_1 \in \mathcal{F}_{d_1,K_2}(\max \left\{\calW_1,\calW_2\right\},\calD_1+\calD_2,\calB_2\max\{\calB_1,1\})$.
    \\
    (3) If $d_1=d_2$ and $K_1=K_2$, then 
    $c_1g_1 + c_2g_2 \in \mathcal{F}_{d_1,K_1}(\calW_1+\calW_2,\max \left\{\calD_1,\calD_2\right\},\left | c_1 \right | \calB_1+\left | c_2 \right | \calB_2),\forall c_1,c_2\in \mathbb{R}$.\\
    (4) If $d_1=d_2$, define $g(x)=(g_1(x),g_2(x))$,   then $$g\in \calF_{d_1+d_2,K_1+K_2}(\calW_1+\calW_2,\max\{\calD_1,\calD_2\},\max\{\calB_1,\calB_2\}).$$
\end{lemma}

\begin{lemma}[Lemma A.8 in \cite{PETERSEN2018}]\label{PetersenA.8}
    Let $\beta\in(0,\infty)$,  $\beta=r+\sigma$ with $r \in \mathbb{N}$ and $\sigma \in (0,1]$ , let $d\in \mathbb{N}^+$. 
    For each $f\in \calF_{\tau,d,1,V}$ and arbitrary $x_0\in(0,1)^d$, there is a polynomial $p(x)= {\textstyle \sum_{\left | \alpha  \right |\le r }c_{\alpha}( x-x_0)^\alpha }$ with $c_\alpha \in [-C \cdot V,C \cdot V]$ for all $\alpha \in \mathbb{N}^d$ with $\left | \alpha \right |\le r $ and such that
    \begin{equation*}
        \left | f(x)-p(x) \right | \le C \cdot V \cdot  \left | x-x_0 \right | ^\beta, \quad \text{for all $x\in [0,1]^d$},
    \end{equation*}
    where $C=d^r$.
\end{lemma}

\begin{lemma}[Lemma 27 in \cite{feng2023}]\label{lem:high dimension}
    For any $d \ge 2$ and $k\in \mathbb{N}$, there exists $m \in \calF_{d,1}((6K+3)d,(k+1)d,d^7(2d)^{k+1})$ such that:
    \begin{equation*}
        \left | x_1\cdots x_d-m(x) \right | \le 3d2^{-2k},\quad x=(x_1,\cdots,x_d)^\top \in[-1,1]^d.
    \end{equation*}
    Moreover, $m(x)=0$ if $x_1\cdots x_d=0$.
\end{lemma}

\section{The Lipschitz-Type Property}\label{appendix:lipschitz}
We start from the Lipschitz-type property for the $\ell_p,p\ge1$, CE and reverse CE losses.
\begin{lemma}\label{lem:lp}
    For the $\ell_p,p\ge1$ loss function and $a,b\in\R^K$, let $\tilde{a}=\mathrm{softmax}(a)$ and $\tilde{b}=\mathrm{softmax}(b)$, we have
    \[\left | \ell(\tilde a,q)-\ell(\tilde b,q) \right | 
       \le \sqrt{2}K\|a-b\|_2,\ \forall q\in\calY.\]
\end{lemma}
\begin{proof}
    Let the $i$-th softmax function be $S(x)_i=\frac{e^{x_i}}{\sum_{j}e^{x_j}}$,  and $s=\sum_{l}e^{x_l}$, then we have
    \begin{align*}
    \|\nabla_x [S(x)_i]\|_2 =&\sqrt{\frac{(\sum_{l\neq i}e^{x_i}e^{x_l})^2+\sum_{l \neq i}e^{2x_i}e^{2x_l}}{s^4}}\le \sqrt{2 \frac{(\sum_{l\neq i}e^{x_i}e^{x_l})^2}{s^4}}\\
    \le& \sqrt{2 \frac{(\sum_{l\neq i}e^{x_i}e^{x_l})^2}{(\sum_{l\neq i}e^{x_i}e^{x_l})^2}}=\sqrt{2},
\end{align*}
 where the first equality used the fact that $\sum_{i}x_i^2\le (\sum_{i}x_i)^2$. Then we have
 \[\|\tilde{a}-\tilde{b}\|_1\le \sum_{i=1}^{K} \left |  \tilde{a}_i- \tilde{b}_i \right | 
       =\sum_{i=1}^K \|\nabla_x[ S(\xi)_i]\|_2 \left \| a-b \right \|_2
        \le \sqrt{2}K\|a-b\|_2.\]
Therefore, for any $p\ge1$, we have $| \ell(\tilde a,q)-\ell(\tilde b,q) |\le\ell(\tilde{a},\tilde{b})\le\|\tilde{a}-\tilde{b}\|_1\le\sqrt{2}K\|a-b\|_2.$
\end{proof}

\begin{lemma}\label{lem:lipschitz}
    For the CE loss function and any $a,b\in\R^K$, let $\tilde{a}=\mathrm{softmax}(a)$ and $\tilde{b}=\mathrm{softmax}(b)$, we have
    \[\left | \ell(\tilde a,q)-\ell(\tilde b,q) \right | 
       \le \sqrt{2}K\|a-b\|_2,\ \forall q\in\calY.\]
\end{lemma}
\begin{proof}
   Let the $i$-th softmax-CE loss function be $\mathrm{CE}(x)_i=-\log(\frac{e^{x_i}}{\sum_{j}e^{x_j}})$,  and $s=\sum_{l}e^{x_l}$, then we have
  

\begin{align*}
    \|\nabla_x [\mathrm{CE}(x)_i]\|_2 =\sqrt{\frac{(\sum_{l\neq i}e^{x_l})^2+\sum_{l \neq i}e^{2x_l}}{s^2}}\le \sqrt{2 \frac{({\sum_{l \neq i}e^{x_l})^2}}{s^2}}= \sqrt{2 \frac{(\sum_{l \neq i}e^{x_l})^2}{(\sum_{l \neq i}e^{x_l})^2}}=\sqrt{2},
\end{align*}
where the first equality used the fact that $\sum_{i}x_i^2\le (\sum_{i}x_i)^2$. For CE loss function $\ell(\tilde{a},q)=-\sum_{i=1}^{K} q_i\log \tilde{a}_i$, $\ell(\tilde{b},q)=-\sum_{i=1}^{K} q_i\log \tilde{b}_i$, there exists $\exists \xi=[\xi_1\cdots\xi_K]^\top$ with $\xi_i \in  [a_i,b_i],\forall 1\le i\le K$ such that
     \begin{align*}
       \left | \ell(\tilde{a},q)-\ell(\tilde{b},q) \right | 
       &=\left | \sum_{i=1}^{K} q_i\log \tilde{a}_i-\sum_{i=1}^{K} q_i\log \tilde{b}_i\right | =\left | \sum_{i=1}^K q_i[\log \tilde{a}_i-\log \tilde{b}_i]\right |\\
       & \le \sum_{i=1}^{K} \left | q_i \right |\left | \log \tilde{a}_i-\log \tilde{b}_i \right | 
       =\sum_{i=1}^K\left | q_i \right | \|\nabla [\mathrm{CE}(\xi)_i]\|_2 \| a-b\|_2\\
       & \le \|q\|_2 \sqrt{2}K\|a-b\|_2\le\sqrt{2}K\|a-b\|_2.
   \end{align*}
\end{proof}

\begin{definition}[The reverse CE \cite{wang2019symmetric}]\label{appendix:def_reversece}
    For all $\tilde{a},q\in\calY=\{y=[y_1\ \cdots\ y_K]^\top \in\R^K|y_i\in[0,1],\sum^K_{i=1}y_i=1\}$ and a negative constant $-\infty<A<0$, the reverse CE loss function is defined as 
    \[\ell(\tilde{a},q)=-\sum_{i=1}^{K} \tilde{a}_ih(q_i),\]
    where
    \[h(x)=\left\{
    \begin{array}{cl}
    \log x,     &  x>e^A,\\
    A,     & \mathrm{otherwise}.
    \end{array}
    \right.\]
\end{definition}
\begin{lemma}\label{lem:lipschitz_rce}
    For any $a,b\in\R^K$, let $\tilde{a}=\mathrm{softmax}(a)$ and $\tilde{b}=\mathrm{softmax}(b)$, we have for the reverse CE loss function
    \[\left | \ell(\tilde{a},q)-\ell(\tilde{b},q) \right | 
       \le-\sqrt{2}KA \|a-b\|_2,\ \forall q\in\calY.\]
\end{lemma}
\begin{proof}
    For the reverse CE loss function $\ell(\tilde{a},q)=-\sum_{i=1}^{K} \tilde{a}_ih(q_i)$, $\ell(\tilde{b},q)=-\sum_{i=1}^{K} \tilde{b}_ih(q_i)$, there exists $\xi \in  [\tilde{a},\tilde{b}]$, we directly calculate
\begin{align*}
      & \left | \ell(\tilde{a},q)-\ell(\tilde{b},q) \right | 
       =\left | \sum_{i=1}^K \tilde{a}_ih(q_i)-\sum_{i=1}^K \tilde{b}_ih(q_i)\right | =\left | \sum_{i=1}^K h(q_i)(\tilde{a}_i-\tilde{b}_i)\right| \\
         \le& \sum_{i=1}^K \left | h(q_i) \right |\left | \tilde{a}_i-\tilde{b}_i \right |   \le \max_i|h(q_i)|\|\tilde{a}-\tilde{b}\|_1\le-A\|\tilde{a}-\tilde{b}\|_1 \le -\sqrt{2}KA \|a-b\|_2,
   \end{align*}
   where the last inequality is from Lemma \ref{lem:lp}.
\end{proof}
\end{document}